\documentclass[11pt, notitlepage]{article}

\usepackage[english]{babel}
\usepackage[paper=letterpaper,top=2.5cm,bottom=2.5cm,inner=2.5cm,outer=2.5cm]{geometry}
\usepackage{setspace}

\usepackage{libertine}
\usepackage{libertinust1math}
\usepackage[T1]{fontenc}

\usepackage{url}  

\usepackage{amsmath}
\usepackage{amsthm}
\usepackage{amssymb}
\usepackage{cases}
\usepackage{proof}

\usepackage{hyperref}

\usepackage{algorithm}
\usepackage{algpseudocode}
\makeatletter
\def\BState{\State\hskip-\ALG@thistlm}
\makeatother

\usepackage{color}

\usepackage{enumerate}

\usepackage{wrapfig}

\usepackage{bm, bbm}

\usepackage{subfigure}

\usepackage{booktabs}

\usepackage{multicol}

\usepackage{pgf,tikz}
\usetikzlibrary{shapes,automata,snakes,backgrounds,arrows}
\usetikzlibrary{mindmap}

\def\squareforqed{\hbox{\rlap{$\sqcap$}$\sqcup$}}
\def\qed{\ifmmode\squareforqed\else{\unskip\nobreak\hfil
\penalty50\hskip1em\null\nobreak\hfil\squareforqed
\parfillskip=0pt\finalhyphendemerits=0\endgraf}\fi}

\newtheorem{theorem}{Theorem}
\newtheorem{definition}{Definition}

\newtheorem{lemma}{Lemma}
\newtheorem{corollary}{Corollary}
\newtheorem{remark}{Remark}

\newtheorem{assumption}{Assumption}

\newcommand{\real}[1]{\mathbb{R}^{{#1}}}
\newcommand{\innerprod}[1]{\langle {#1} \rangle}
\newcommand{\norm}[1]{\lVert {#1} \rVert}

\newcommand{\Norm}[1]{\left\lVert {#1} \right\rVert}

\newcommand{\prox}[2][{}]{\mathbf{prox}^{#1}_{#2}}

\usepackage{natbib}
\setcitestyle{authoryear, open={(},close={)}}

\begin{document}

\title{Asynchronous Stochastic Proximal Methods for \\ Nonconvex Nonsmooth Optimization}

\author{
Rui Zhu$^1$\thanks{rzhu3@ualberta.ca},
Di Niu$^1$\thanks{dniu@ualberta.ca},
Zongpeng Li$^2$\thanks{zongpeng@whu.edu.cn} \\
$^1$ Department of Electrical and Computer Engineering, University of Alberta\\
$^2$ School of Computer Science, Wuhan University
}

\maketitle


\begin{abstract}

We study stochastic algorithms for solving nonconvex optimization problems with a convex yet possibly nonsmooth regularizer, which find wide applications in many practical machine learning applications. However, compared to asynchronous parallel stochastic gradient descent (AsynSGD), an algorithm targeting smooth optimization, the understanding of the behavior of stochastic algorithms for nonsmooth regularized optimization problems is limited, especially when the objective function is nonconvex. To fill this theoretical gap, in this paper, we propose and analyze asynchronous parallel stochastic proximal gradient (Asyn-ProxSGD) methods for nonconvex problems. We establish an ergodic convergence rate of $O(1/\sqrt{K})$ for the proposed Asyn-ProxSGD, where $K$ is the number of updates made on the model, matching the convergence rate currently known for AsynSGD (for smooth problems). To our knowledge, this is the first work that provides convergence rates of asynchronous parallel ProxSGD algorithms for nonconvex problems. Furthermore, our results are also the first to show the convergence of any stochastic proximal methods without assuming an increasing batch size or the use of additional variance reduction techniques. We implement the proposed algorithms on Parameter Server and demonstrate its convergence behavior and near-linear speedup, as the number of workers increases,  on two real-world datasets.

\end{abstract}

\section{Introduction}
\label{sec:intro}

With rapidly growing data volumes and variety, the need to scale up machine learning has sparked broad interests in developing efficient parallel optimization algorithms. 
A typical parallel optimization algorithm usually decomposes the original problem into multiple subproblems, each handled by a worker node. Each worker iteratively downloads the global model parameters and computes its local gradients to be sent to the master node or servers for model updates. Recently, asynchronous parallel optimization algorithms  \citep{recht2011hogwild,li2014communication,lian2015asynchronous}, exemplified by the Parameter Server architecture \citep{li2014scaling}, have been widely deployed in industry to solve practical large-scale machine learning problems. Asynchronous algorithms can largely reduce overhead and speedup training, since each worker may individually perform model updates in the system without synchronization.
Another trend to deal with large volumes of data is the use of \emph{stochastic} algorithms.  
As the number of training samples $n$ increases, the cost of updating the model $x$ taking into account all error gradients becomes prohibitive. To tackle this issue, stochastic algorithms make it possible to update $x$ using only a small subset of all training samples at a time.


Stochastic gradient descent (SGD) is one of the first algorithms widely implemented in an asynchronous parallel fashion; its convergence rates and speedup properties have been analyzed for both convex \citep{agarwal2011distributed,mania2017} and nonconvex \citep{lian2015asynchronous} optimization problems.
Nevertheless, SGD is mainly applicable to the case of smooth optimization, and yet is not suitable for problems with a \emph{nonsmooth} term in the objective function, e.g., an $\ell_1$ norm regularizer. In fact, such nonsmooth regularizers are commonplace in many practical machine learning problems or constrained optimization problems. In these cases, SGD becomes ineffective, as it is hard to obtain gradients for a nonsmooth objective function.


We consider the following nonconvex regularized optimization problem:
{
\begin{equation}
\begin{split}
	\mathop{\min}_{x\in \real{d}} &\quad \Psi(x):= f(x) + h(x),
\end{split}
\label{eq:original}
\end{equation}
}
where $f(x)$ takes a finite-sum form of $f(x) := \frac{1}{n}\sum_{i=1}^n f_i(x)$, and each $f_i(x)$ is a smooth (but not necessarily convex) function. The second term $h(x)$ is a convex (but \emph{not necessarily smooth}) function. This type of problems is prevalent in machine learning, as exemplified by deep learning with regularization \citep{dean2012large,chen2015mxnet,zhang2015deep}, LASSO \citep{tibshirani2005sparsity}, sparse logistic regression \citep{liu2009large}, robust matrix completion \citep{xu2010robust,sun2015guaranteed}, and sparse support vector machine (SVM) \citep{friedman2001elements}. In these problems, $f(x)$ is a loss function of model parameters $x$, possibly in a nonconvex form (e.g., in neural networks), while $h(x)$ is a convex regularization term, which is, however, possibly \emph{nonsmooth}, e.g., the $\ell_1$ norm regularizer.

Many classical deterministic (non-stochastic) algorithms are available to solve problem \eqref{eq:original}, including the proximal gradient (ProxGD) method \citep{parikh2014proximal} and its accelerated variants \citep{li2015accelerated} as well as the alternating direction method of multipliers (ADMM) \citep{hong2016convergence}. These methods leverage the so-called \emph{proximal operators} \citep{parikh2014proximal} to handle the nonsmoothness in the problem. Although implementing these deterministic algorithms in a \emph{synchronous} parallel fashion is straightforward, extending them to asynchronous parallel algorithms is much more complicated than it appears. In fact, existing theory on the convergence of asynchronous proximal gradient (PG) methods for nonconvex problem \eqref{eq:original} is quite limited. An asynchronous parallel proximal gradient method has been presented in \citep{li2014communication} and has been shown to converge to stationary points for nonconvex problems. However, \citep{li2014communication} has essentially proposed a non-stochastic algorithm and has not provided its convergence rate.

In this paper, we propose and analyze an asynchronous parallel \emph{proximal stochastic gradient descent} (ProxSGD) method for solving the nonconvex and nonsmooth problem \eqref{eq:original}, with provable convergence and speedup guarantees. The analysis of ProxSGD has attracted much attention in the community recently.
Under the assumption of an \emph{increasing} minibatch size used in the stochastic algorithm, the non-asymptotic convergence of ProxSGD to stationary points has been shown in \citep{ghadimi2016mini} for problem \eqref{eq:original} with a convergence rate of $O(1/\sqrt{K})$, $K$ being the times the model is updated. Moreover, additional variance reduction techniques have been introduced \citep{reddi2016proximal} to guarantee
the convergence of ProxSGD, which is different from the stochastic method we discuss here. The stochastic algorithm considered in this paper assumes that each worker selects a minibatch of randomly chosen training samples to calculate the gradients at a time, which is a scheme widely used in practice.
To the best of our knowledge, the convergence behavior of ProxSGD---under a \emph{constant} minibatch size without variance reduction---is still unknown (even for the synchronous or sequential version).

Our main contributions are summarized as follows:
\begin{itemize}
	\item We propose asynchronous parallel ProxSGD (a.k.a. Asyn-ProxSGD) and prove that it can converge to stationary points of nonconvex and nonsmooth problem \eqref{eq:original} with an ergodic convergence rate of $O(1/\sqrt{K})$, where $K$ is the number of times that the model $x$ is updated.  This rate matches the convergence rate known for asynchronous SGD. The latter, however, is suitable only for smooth problems. To our knowledge, this is the first work that offers convergence rate guarantees for any stochastic proximal methods in an asynchronous parallel setting.
		
	\item Our result also suggests that the sequential (or synchronous parallel) ProxSGD can converge to stationary points of problem \eqref{eq:original}, with a convergence rate of $O(1/\sqrt{K})$. To the best of our knowledge, this is also the first work that provides convergence rates of any \emph{stochastic} algorithm for nonsmooth problem \eqref{eq:original} under a \emph{constant} batch size, while prior literature on such stochastic proximal methods assumes an increasing batch size or relies on variance reduction techniques.

	\item We provide a linear speedup guarantee as the number of workers increases, provided that the number of workers is bounded by $O({K}^{1/4})$. This result has laid down a theoretical ground for the scalability and performance of our Asyn-ProxSGD algorithm in practice.
\end{itemize}


\section{Preliminaries}
\label{sec:prelim}


In this paper, we use $f(x)$ as the one defined in \eqref{eq:original}, and $F(x;\xi)$ as a function whose stochastic nature comes from the random variable $\xi$ representing a random index selected from the training set $\{1,\ldots, n\}$. We use $\norm{x}$ to denote the $\ell_2$ norm of the vector $x$, and $\innerprod{x, y}$ to denote the inner product of two vectors $x$ and $y$. We use $g(x)$ to denote the ``true'' gradient $\nabla f(x)$ and use $G(x;\xi)$ to denote the stochastic gradient $\nabla F(x;\xi)$ for a function $f(x)$. 
For a random variable or vector $X$, let $\mathbb{E}[X|\mathcal{F}]$ be the conditional expectation of $X$ w.r.t. a sigma algebra $\mathcal{F}$. We denote $\partial h(x)$ as the \emph{subdifferential} of $h$. A point $x$ is a critical point of $\Phi$, iff $0 \in \nabla f(x) + \partial h(x)$. 

\subsection{Stochastic Optimization Problems}
In this paper, we consider the following \emph{stochastic} optimization problem instead of the original deterministic version \eqref{eq:original}:
\begin{equation}
\begin{split}
	\mathop{\min}_{x\in \real{d}} &\quad \Psi(x):= \mathbb{E}_\xi [F(x; \xi)] + h(x),
\end{split}
\label{eq:stochastic}
\end{equation}
where the stochastic nature comes from the random variable $\xi$, which in our problem settings, represents a random index selected from the training set $\{1, \ldots, n\}$. Therefore, \eqref{eq:stochastic} attempts to minimize the expected loss of a random training sample plus a regularizer $h(x)$. In this work, we assume the function $h$ is proper, closed and convex, yet \emph{not necessarily smooth}.

\subsection{Proximal Gradient Descent}

The proximal operator is fundamental to many algorithms to solve problem \eqref{eq:original} as well as its stochastic variant \eqref{eq:stochastic}.

\begin{definition}[Proximal operator]
	The proximal operator $\prox{}$ of a point $x \in \real{d}$ under a proper and closed function $h$ with parameter $\eta > 0$ is defined as:
\begin{equation}
	\prox{\eta h}(x) = \mathop{\arg \min}_{y \in \real{d}} \left\{h(y) + \frac{1}{2\eta} \norm{y-x}^2\right\}.
\end{equation}
\end{definition}
In its vanilla version, \emph{proximal gradient descent} performs the following iterative updates:
\begin{equation*}
	x^{k+1} \gets \prox{\eta_k h}(x^{k} - \eta_k \nabla f(x^k)),
\end{equation*}
for $k=1,2,\ldots$, where $\eta_k > 0$ is the step size at iteration $k$. 

To solve stochastic optimization problem \eqref{eq:stochastic}, we need a variant called \emph{proximal stochastic gradient descent} (ProxSGD), with its update rule at each (synchronized) iteration $k$ given by
\begin{equation}
	x^{k+1} \gets \prox{\eta_k h}\left( x^{k} - \frac{\eta_k}{N}\sum_{\xi \in \Xi_k}\nabla F(x^k; \xi) \right),
	\label{eq:proxsgd_step}
\end{equation}
where $N:=|\Xi_k|$ is the mini-batch size. In ProxSGD, the aggregate gradient $\nabla f$ over all the samples is replaced by the gradients from a random subset of training samples, denoted by $\Xi_k$ at iteration $k$. Since $\xi$ is a random variable indicating a random index in $\{1,\ldots, n\}$, $F(x; \xi)$ is a random loss function for the random sample $\xi$, such that $f(x) := \mathbb{E}_\xi [F(x; \xi)]$.

\subsection{Parallel Stochastic Optimization}

Recent years have witnessed rapid development of parallel and distributed computation frameworks for large-scale machine learning problems. One popular architecture is called \emph{parameter server} \citep{dean2012large,li2014scaling}, which consists of some worker nodes and server nodes. In this architecture, one or multiple master machines play the role of parameter servers, which maintain the model $x$. Since these machines serve the same purpose, we can simply treat them as one \emph{server node} for brevity. All other machines are \emph{worker nodes} that communicate with the server for training machine learning models. In particular, each worker has two types of requests: \textbf{pull} the current model $x$ from the server, and \textbf{push} the computed gradients to the server.

Before proposing an asynchronous Proximal SGD algorithm in the next section, let us first introduce its \emph{synchronous} version.  Let us use an example to illustrate the idea. Suppose we execute ProxSGD with a mini-batch of 128 random samples on 8 workers. We can let each worker randomly take 16 samples, and compute a summed gradient on these 16 samples, and push it to the server. In the synchronous case, the server will finally receive 8 summed gradients (containing information of all 128 samples) in each iteration. The server then updates the model by performing the proximal gradient descent step. In general, if we have $m$ workers, each worker will be assigned $N/m$ random samples in an iteration.

Note that in this scenario, all workers contribute to the computation of the sum of gradients on $N$ random samples in parallel, which corresponds to \emph{data parallelism} in the literature (e.g., \citep{agarwal2011distributed,ho2013more}). Another type of parallelism is called \emph{model parallelism}, in which each worker uses all $N$ random samples in the batch to compute a partial gradient on a specific block of $x$ (e.g., \citep{recht2011hogwild,pan2016cyclades}). Typically, data parallelism is more suitable when $n \gg d$, i.e., large dataset with moderate model size, and model parallelism is more suitable when $d \gg n$. We focus on data parallelism.

\begin{algorithm}[thpb]
\caption{Asyn-ProxSGD: Asynchronous Proximal Stochastic Gradient Descent}
\label{alg:asyn-prox-sgd}
\underline{\textbf{Server executes:}}
\begin{algorithmic}[1]
	\State Initialize $x^0$.
	\State Initialize $G \gets 0$. \Comment {Gradient accumulator}
	\State Initialize $s \gets 0$. \Comment {Request counter}
	\Loop
		\If {\texttt{Pull Request} from worker $j$ is received:}
			\State Send $x$ to worker $j$.
		\EndIf
		\If {\texttt{Push Request} (gradient $G_j$) from worker $j$ is received:}
			\State $s \gets s + 1$.
			\State $G \gets G + \frac{1}{N}\cdot G_j$.
			\If {$s=m$}
				\State $x \gets \prox{\eta h}(x - \eta G)$.
				\State $s \gets 0$.
				\State $G \gets 0$.
			\EndIf
		\EndIf
	\EndLoop
\end{algorithmic}
\underline{\textbf{Worker $j$ asynchronously performs:}}
\begin{algorithmic}[1]
\State Pull $x^0$ to initialize.
\For {$t=0,1,\ldots$}
	\State Randomly choose $N/m$ training samples indexed by $\xi_{t,1}(j),\ldots,\xi_{t,N/m}(j)$.
	\State Calculate $G^t_j = \sum_{i=1}^{N} \nabla F(x^t; \xi_{t, i}(j))$.
	\State Push $G^t_j$ to the server.
	\State Pull the current model $x$ from the server: $x^{t+1} \gets x$.
\EndFor
\end{algorithmic}
\end{algorithm}
\vspace{-3mm}

\section{Asynchronous Proximal Gradient Descent}
\label{sec:aspg}

We now present our \emph{asynchronous proximal gradient descent} (Asyn-ProxSGD) algorithm, which is the main contribution in this paper. In the asynchronous algorithm, different workers may be in different local iterations due to random delays in computation and communication.

For ease of presentation, let us first assume each worker uses only one random sample at a time to compute its stochastic gradient, which naturally generalizes to using a mini-batch of random samples to compute a stochastic gradient. In this case, each worker will independently and asynchronously repeat the following steps:
\begin{itemize}
\item Pull the latest model $x$ from the server;
\item Calculate a gradient $\tilde{G}(x;\xi)$ based on a random sample $\xi$ locally;
\item Push the gradient $\tilde{G}(x; \xi)$ to the server.
\end{itemize}
Here we use $\tilde{G}$ to emphasize that the gradient computed on workers \emph{may be delayed}. For example, all workers but worker $j$ have completed their tasks of iteration $t$, while worker $j$ still works on iteration $t-1$. In this case, the gradient $\tilde{G}$ is not computed based on the current model $x^t$ but from a delayed one $x^{t-1}$.

In our algorithm, the server will perform an averaging over the received sample gradients as long as $N$ gradients are received and perform an proximal gradient descent update on the model $x$, no matter where these $N$ gradients come from; as long as $N$ gradients are received, the averaging is performed. This means that it is possible that the server may have received multiple gradients from one worker while not receiving any from another worker. 

In general, when each mini-batch has $N$ samples, and each worker processes $N/m$ random samples to calculate a stochastic gradient to be pushed to the server, the proposed Asyn-ProxSGD algorithm is described in Algorithm~\ref{alg:asyn-prox-sgd} leveraging a parameter server architecture. The server maintains a counter $s$. Once $s$ reaches $m$, the server has received gradients that contain information about $N$ random samples (no matter where they come from) and will perform a proximal model update.



\section{Convergence Analysis}
\label{sec:theory}

To facilitate the analysis of Algorithm~\ref{alg:asyn-prox-sgd}, we rewrite it in an equivalent global view (from the server's perspective), as described in Algorithm~\ref{alg:apsgd-global}. In this algorithm, we use an iteration counter $k$ to keep track of how many times the model $x$ has been updated on the server; $k$ increments every time a push request (model update request) is completed. Note that such a counter $k$ is \emph{not} required by workers to compute gradients and is different from the counter $t$ in Algorithm~\ref{alg:asyn-prox-sgd}---$t$ is maintained by each worker to count how many sample gradients have been computed locally.

In particular, for every $N$ stochastic sample gradients received, the server simply aggregates them by averaging:  
{
\begin{equation}
	\tilde{G}^k := \frac{1}{N}\sum_{i=1}^N \nabla F(x^{k-\tau(k,i)}; \xi_{k, i}),
	\label{eq:grad_aspg}
\end{equation}
}
where $\tau(k, i)$ indicates that the stochastic gradient $\nabla F(x^{k-\tau(k,i)}; \xi_{k, i})$ received at iteration $k$ could have been computed based on an older model $x^{k-\tau(k,i)}$ due to communication delay and asynchrony among workers. Then, the server updates $x^k$ to $x^{k+1}$ using proximal gradient descent.

\begin{algorithm}[htpb]
\caption{Asyn-ProxSGD (from a Global Perspective)}
\label{alg:apsgd-global}
\begin{algorithmic}[1]
\State Initialize $x^1$.
\For {$k=1,\ldots, K$}
\State {Randomly select $N$ training samples indexed by $\xi_{k, 1}, \ldots, \xi_{k, N}$.}
\State {Calculate the averaged gradient $\tilde{G}^k$ according to \eqref{eq:grad_aspg}.}
\State {$x^{k+1} \gets \prox{\eta_k h}(x^{k} - \eta_k \tilde{G}^{k})$.}
\EndFor
\end{algorithmic}
\end{algorithm}

\subsection{Assumptions and Metrics}
We make the following assumptions for convergence analysis. We assume that $f(\cdot)$ is a smooth function with the following properties:
\begin{assumption}[Lipschitz Gradient]
	For function $f$ there are Lipschitz constants $L>0$ such that
	{
	\begin{equation}
		\norm{\nabla f(x) - \nabla f(y)} \leq L \norm{x - y}, \forall x, y \in \real{d}.
	\end{equation}
	}
\label{asmp:smooth}
\vspace{-2mm}
\end{assumption}
As discussed above, assume that $h$ is a proper, closed and convex function, which is yet not necessarily smooth. 
If the algorithm has been executed for $k$ iterations, we let $\mathcal{F}_k$ denote the set that consists of all the samples used up to iteration $k$. Since $\mathcal{F}_k \subseteq \mathcal{F}_{k'}$ for all $k \leq k'$, the collection of all such $\mathcal{F}_k$ forms a \emph{filtration}. Under such settings, we can restrict our attention to those stochastic gradients with an unbiased estimate and bounded variance, which are common in the analysis of \emph{stochastic} gradient descent or \emph{stochastic} proximal gradient algorithms, e.g., \citep{lian2015asynchronous,ghadimi2016mini}.
\begin{assumption}[Unbiased gradient]
	For any $k$, we have $\mathbb{E}[G_k | \mathcal{F}_{k}] = g_k$.
	\label{asmp:unbias_grad}
\end{assumption}
\begin{assumption}[Bounded variance]
	The variance of the stochastic gradient is bounded by $\mathbb{E}[\norm{G(x;\xi) - \nabla f(x)}^2] \leq \sigma^2$.
	\label{asmp:var_grad}
\end{assumption}

We make the following assumptions on the delay and independence:
\begin{assumption}[Bounded delay]
	All delay variables $\tau(k,i)$ are bounded by $T$: $\max_{k, i} \tau(k,i) \leq T$.
	\label{asmp:bound1}
\end{assumption}
\begin{assumption}[Independence]
	All random variables $\xi_{k, i}$ for all $k$ and $i$ in Algorithm~\ref{alg:apsgd-global} are mutually independent.
	\label{asmp:indie1}
\end{assumption}

The assumption of bounded delay is to guarantee that gradients from workers should not be too old. Note that the maximum delay $T$ is  roughly \emph{proportional to the number of workers} in practice. This is also known as \emph{stale synchronous parallel} \citep{ho2013more} in the literature. Another assumption on independence can be met by selecting samples with \emph{replacement}, which can be implemented using some distributed file systems like HDFS \citep{borthakur2008hdfs}. These two assumptions are common in convergence analysis for asynchronous parallel algorithms, e.g., \citep{lian2015asynchronous,davis2016sound}.

\subsection{Theoretical Results}
We present our main convergence theorem as follows:
\begin{theorem}
	If Assumptions~\ref{asmp:bound1} and \ref{asmp:indie1} hold and the step length sequence $\{\eta_k\}$ in Algorithm~\ref{alg:apsgd-global} satisfies
	{
	\begin{equation}
		\eta_k \leq \frac{1}{16L},\quad 6\eta_k L^2 T \sum_{l=1}^T \eta_{k+l} \leq 1,
	\end{equation}
	}
	for all $k=1,2,\ldots, K$, we have the following ergodic convergence rate for Algorithm ~\ref{alg:apsgd-global}:
	{
	\begin{equation}
	\begin{split}
&\quad \frac{\sum_{k=1}^K (\eta_k - 8L\eta_k^2) \mathbb{E}[\norm{P(x^k, g^k, \eta_k)}^2]}{\sum_{k=1}^K (\eta_k - 8L\eta_k^2) }  \\
&\leq \frac{8(\Psi(x^1) - \Psi(x^*))}{\sum_{k=1}^K \eta_k-8L\eta_k^2}  + \frac{\sum_{k=1}^K \left(8 L\eta_k^2 + 12\eta_kL^2T \sum_{l=1}^{T}\eta_{k-l}^2 \right) \sigma^2}{N\sum_{k=1}^K (\eta_k - 8L\eta_k^2)},
\end{split}
	\end{equation}
	}
	where the expectation is taken in terms of all random variables in Algorithm~\ref{alg:apsgd-global}.
	\label{thm:aspg_convergence}
\end{theorem}

Taking a closer look at Theorem~\ref{thm:aspg_convergence}, we can properly choose the learning rate $\eta_k$ as a constant value and derive the following convergence rate:
\begin{corollary}
	Let the step length be a constant, i.e., 
	{
	\begin{equation}
	    \eta := \sqrt{\frac{(\Psi(x^1) - \Psi(x^*))N}{2KL\sigma^2}}.
	\end{equation}
	}
	If the delay bound $T$ satisfies
	{
	\begin{equation}\label{eq:T_bound}
	    K \geq \frac{128(\Psi(x^1) - \Psi(x^*))NL}{\sigma^2} (T+1)^4,
	\end{equation}
	}
	then the output of Algorithm~\ref{thm:aspg_convergence} satisfies the following ergodic convergence rate:
	{
	\begin{equation}
	\mathop{\min}_{k=1,\ldots,K} \mathbb{E}[\norm{P(x^k, g^k, \eta_k)}^2] 
	\leq \frac{1}{K}\sum_{k=1}^K \mathbb{E}[\norm{P(x^k, g^k, \eta_k)}^2] 
	\leq 32\sqrt{\frac{2(\Psi(x^1) - \Psi(x^*))L\sigma^2}{KN}}.
	\label{eq:convergence_1}
	\end{equation}
	}
	\label{corr:aspg_convergence}
\end{corollary}

\begin{remark}[Consistency with ProxSGD]
	When $T=0$, our proposed Asyn-ProxSGD reduces to the vanilla ProxSGD (e.g., \citep{ghadimi2016mini}). Thus, the iteration complexity is $O(1/\epsilon^2)$ according to \eqref{eq:convergence_1}, attaining the same result as that in \citep{ghadimi2016mini} \emph{yet without assuming increased mini-batch sizes}.
\end{remark}
\begin{remark}[Linear speedup w.r.t. the staleness]
	From \eqref{eq:convergence_1} we can see that linear speedup is achievable, as long as the delay $T$ is bounded by $O(K^{1/4})$ (if other parameters are constants). The reason is that by \eqref{eq:T_bound} and \eqref{eq:convergence_1}, as long as $T$ is no more than $O(K^{1/4})$, the iteration complexity (from a global perspective) to achieve $\epsilon$-optimality is $O(1/\epsilon^2)$, which is independent from $T$. 
\end{remark}
\begin{remark}[Linear speedup w.r.t. number of workers]
	As the iteration complexity is $O(1/\epsilon^2)$ to achieve $\epsilon$-optimality, it is also independent from the number of workers $m$ if assuming other parameters are constants. It is worth noting that the delay bound $T$ is roughly proportional to the number of workers. As the iteration complexity is independent from $T$, we can conclude that the total iterations will be shortened to $1/T$ of a single worker's iterations if $\Theta(T)$ workers work in parallel, achieving nearly linear speedup.
\end{remark}
\begin{remark}[Comparison with Asyn-SGD]
	Compared with asynchronous SGD \citep{lian2015asynchronous}, in which $T$ or the number of workers should be bounded by $O(\sqrt{K/N})$ to achieve linear speedup, here Asyn-ProxSGD is more sensitive to delays and more suitable for a smaller cluster.
\end{remark}

\section{Experiments}
\label{sec:simu}
\begin{figure*}[t]
  \centering
  \subfigure[\texttt{a9a}]{
    \includegraphics[height=1.6in]{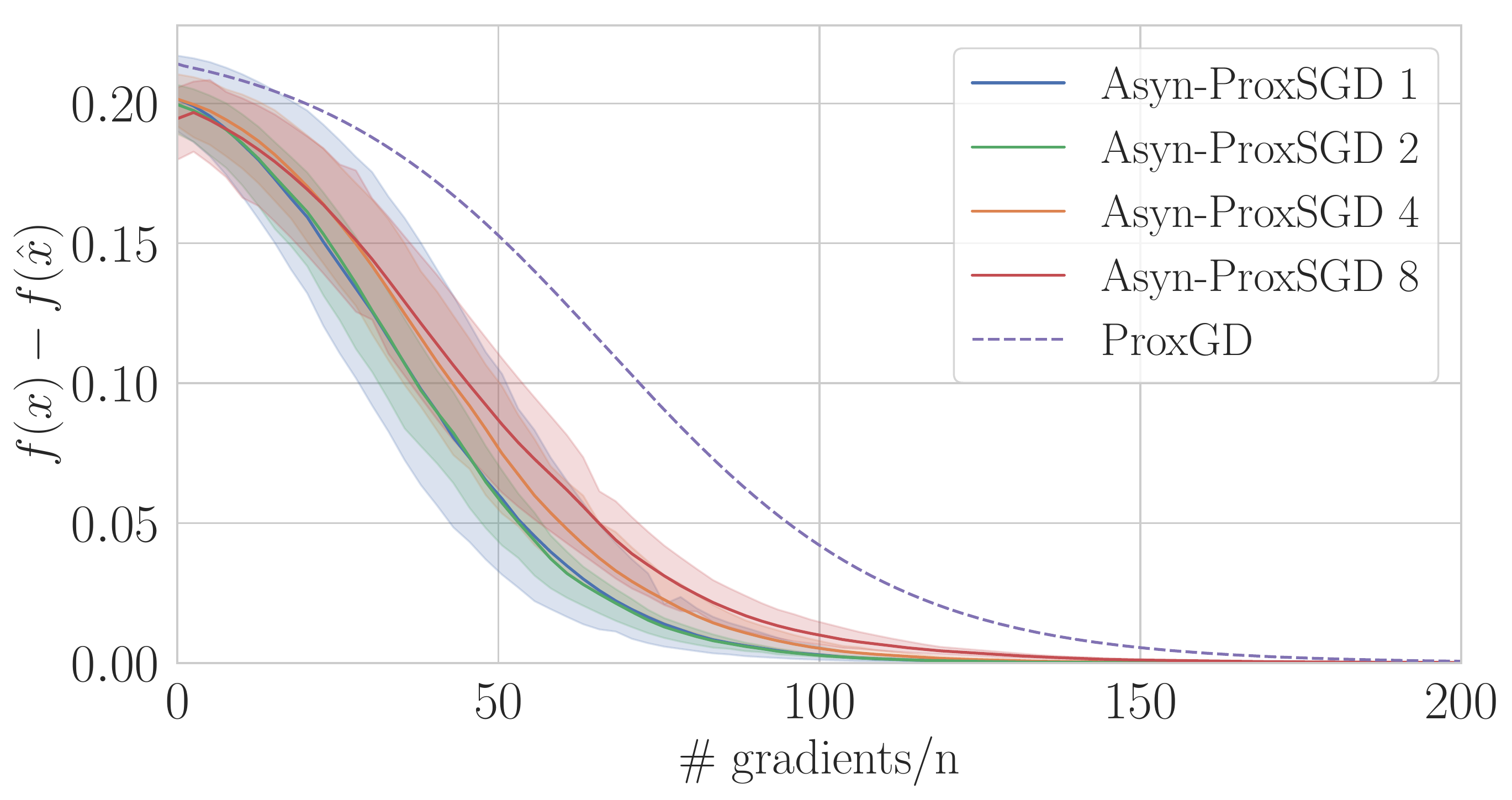}
    \label{fig:obj1-a9a}
  }
  \subfigure[\texttt{mnist}]{
    \includegraphics[height=1.6in]{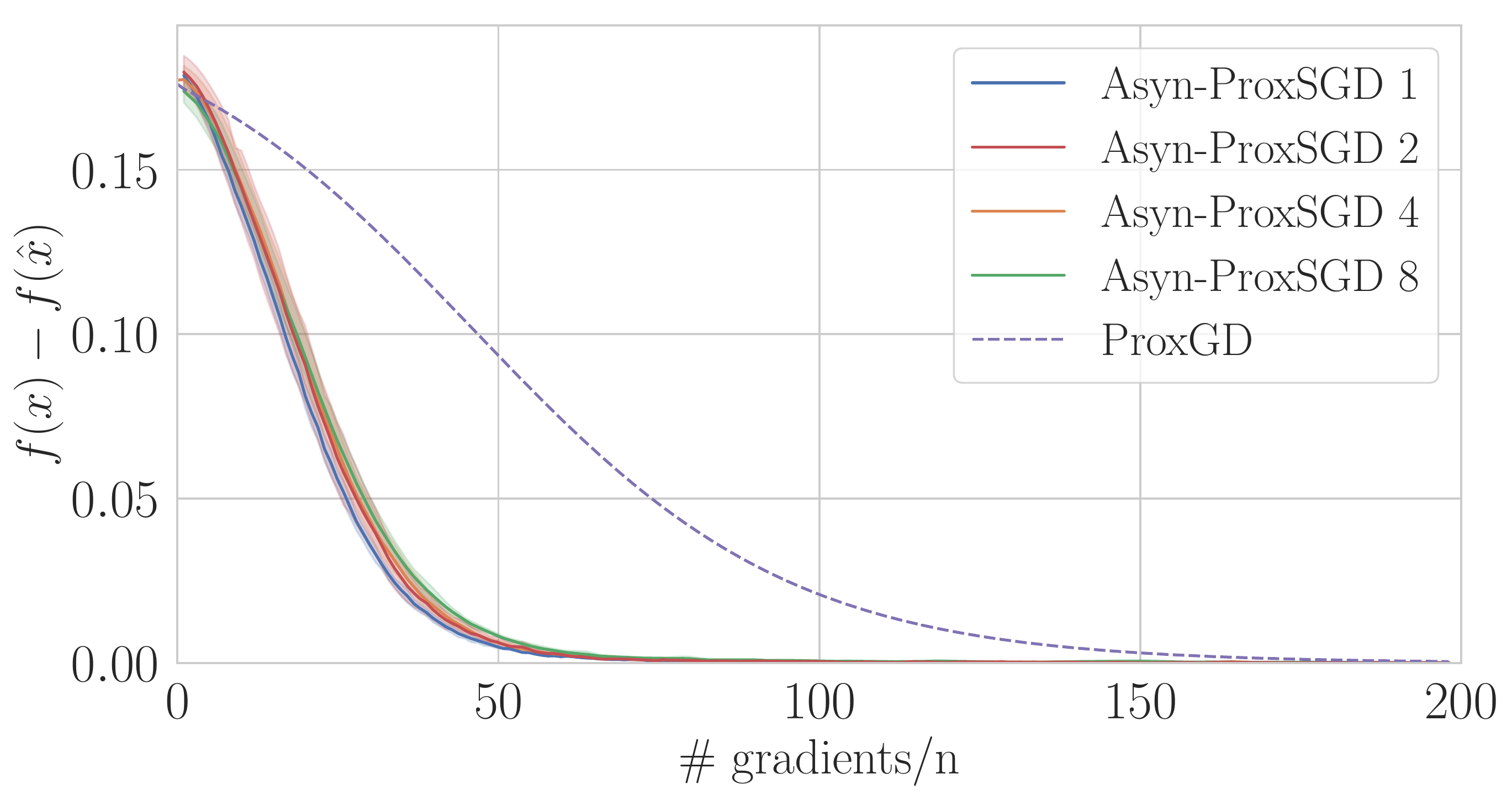}
    \label{fig:obj1-mnist}
  }
  \vspace{-1mm}
  \subfigure[\texttt{a9a}]{
    \includegraphics[height=1.6in]{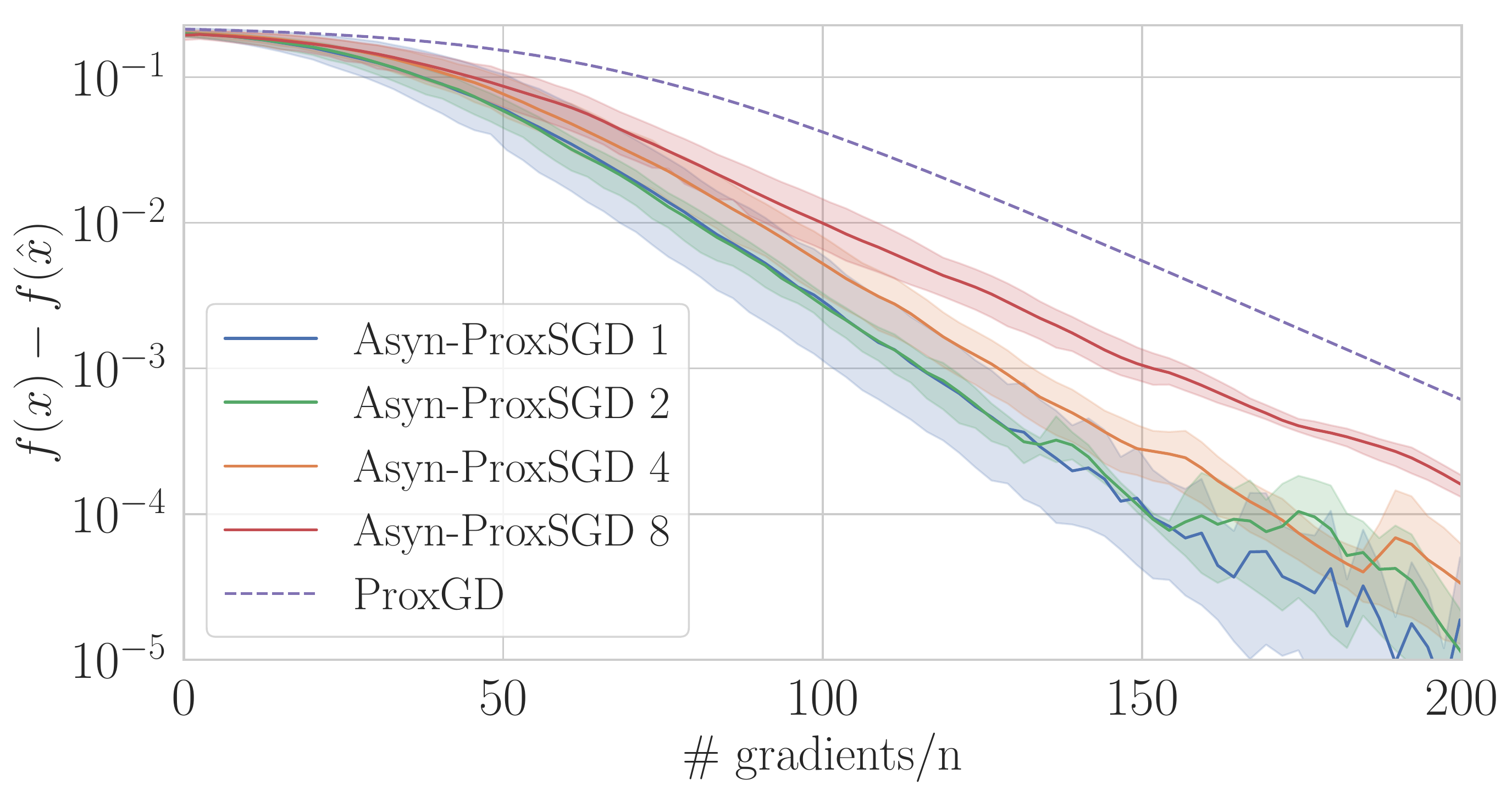}
    \label{fig:obj2-a9a}
  }
  \subfigure[\texttt{mnist}]{
    \includegraphics[height=1.6in]{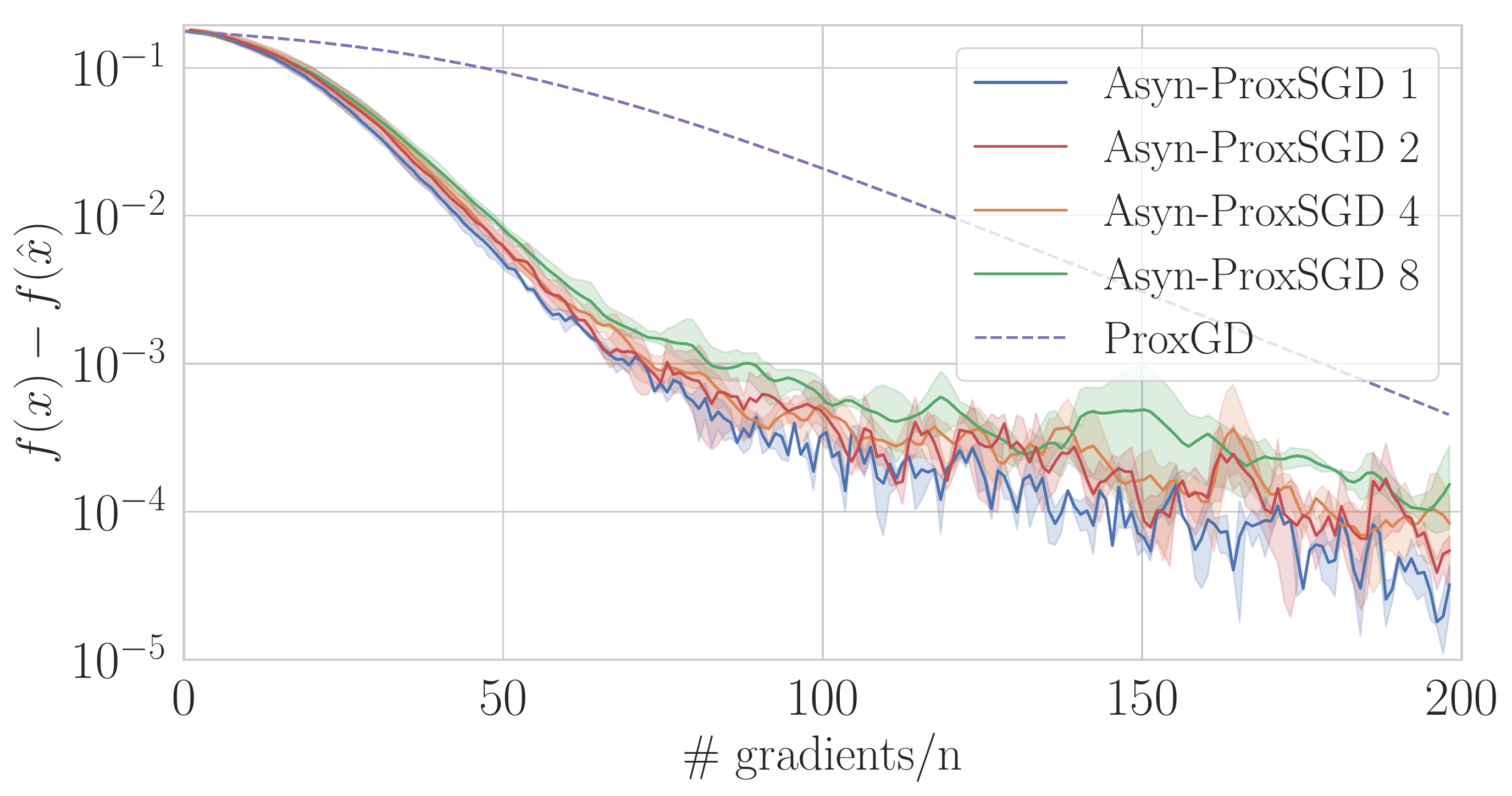}
    \label{fig:obj2-mnist}
  }
  \caption{Performance of ProxGD and Async-ProxSGD on \texttt{a9a} (left) and \texttt{mnist} (right) datasets. Here the x-axis represents how many sample gradients is computed (divided by $n$), and the y-axis is the function suboptimality $f(x) - f(\hat{x})$ where $\hat{x}$ is obtained by running gradient descent for many iterations with multiple restarts. Note all values on the y-axis are normalized by $n$.}
  \label{fig:simu_grad_loss}
\end{figure*}

We now present experimental results to confirm the capability and efficiency of our proposed algorithm to solve challenging non-convex non-smooth machine learning problems. We implemented our algorithm on TensorFlow \citep{abadi2016tensorflow}, a flexible and efficient deep learning library. We execute our algorithm on \texttt{Ray} \citep{moritz2017ray}, a general-purpose framework that enables parallel and distributed execution of Python as well as TensorFlow functions. A key feature of \texttt{Ray} is that it provides a unified task-parallel abstraction, which can serve as workers, and actor abstraction, which stores some states and acts like parameter servers.

We use a cluster of 9 instances on Google Cloud. Each instance has one CPU core with 3.75 GB RAM, running 64-bit Ubuntu 16.04 LTS. Each server or worker uses only one core, with 9 CPU cores and 60 GB RAM used in total. Only one instance is the server node, while the other nodes are workers.

\textbf{Setup:} 
In our experiments, we consider the problem of non-negative principle component analysis (NN-PCA) \citep{reddi2016proximal}. Given a set of $n$ samples $\{z_i\}_{i=1}^n$, NN-PCA solves the following optimization problem
{
\begin{equation}
  \mathop{\min}_{\norm{x} \leq 1, x \geq 0} -\frac{1}{2}x^\top \left(\sum_{i=1}^n z_i z_i^\top \right)x.
\end{equation}
}
This NN-PCA problem is NP-hard in general. To apply our algorithm, we can rewrite it with $f_i(x) = -(x^\top z_i)^2 / 2$ for all samples $i \in [n]$. Since the feasible set $C=\{x \in \real{d} | \norm{x} \leq 1, x \geq 0\}$ is convex, we can replace the optimization constraint by a regularizer in the form of 
an indicator function $h(x) = I_C(x)$ , such that $h(x)=0$ if $x \in C$ and $\infty$ otherwise. 

\begin{table}[]
\centering
\caption{Description of the two classification datasets used.}
\label{table:dataset}
\begin{tabular}{c|cc}
\specialrule{.1em}{.05em}{.05em}
datasets & dimension & sample size \\ \hline
a9a      & 123       & 32,561      \\
mnist    & 780       & 60,000      \\
\specialrule{.1em}{.05em}{.05em}
\end{tabular}
\end{table}


The hyper-parameters are set as follows. The step size is set using the popular $t$-inverse step size choice $\eta_k = \eta_0/(1 + \eta' (k/k'))$, which is the same as the one used in \citep{reddi2016proximal}. Here $\eta_0, \eta' > 0$ determine how learning rates change, and $k'$ controls for how many steps the learning rate would change.

We conduct experiments on two datasets \footnote{Available at \url{http://www.csie.ntu.edu.tw/~cjlin/libsvmtools/datasets/}}, with their information summarized in Table~\ref{table:dataset}. All samples have been normalized, i.e., $\norm{z_i} = 1$ for all $i \in [n]$. In our experiments, we use a batch size of $N=8192$ in order to evaluate the performance and speedup behavior of the algorithm under constant batches.

We consider the \emph{function suboptimality} value as our performance metric. In particular, we run proximal gradient descent (ProxGD) for a large number of iterations with multiple random initializations, and obtain a solution $\hat{x}$. For all experiments, we evaluate function suboptimality, which is
the gap $f(x) - f(\hat{x})$, against 
the number of sample gradients processed by the server (divided by the total number of samples $n$), and then against time.


\begin{figure*}[t]
  \centering
  \subfigure[\texttt{a9a}]{
    \includegraphics[height=1.6in]{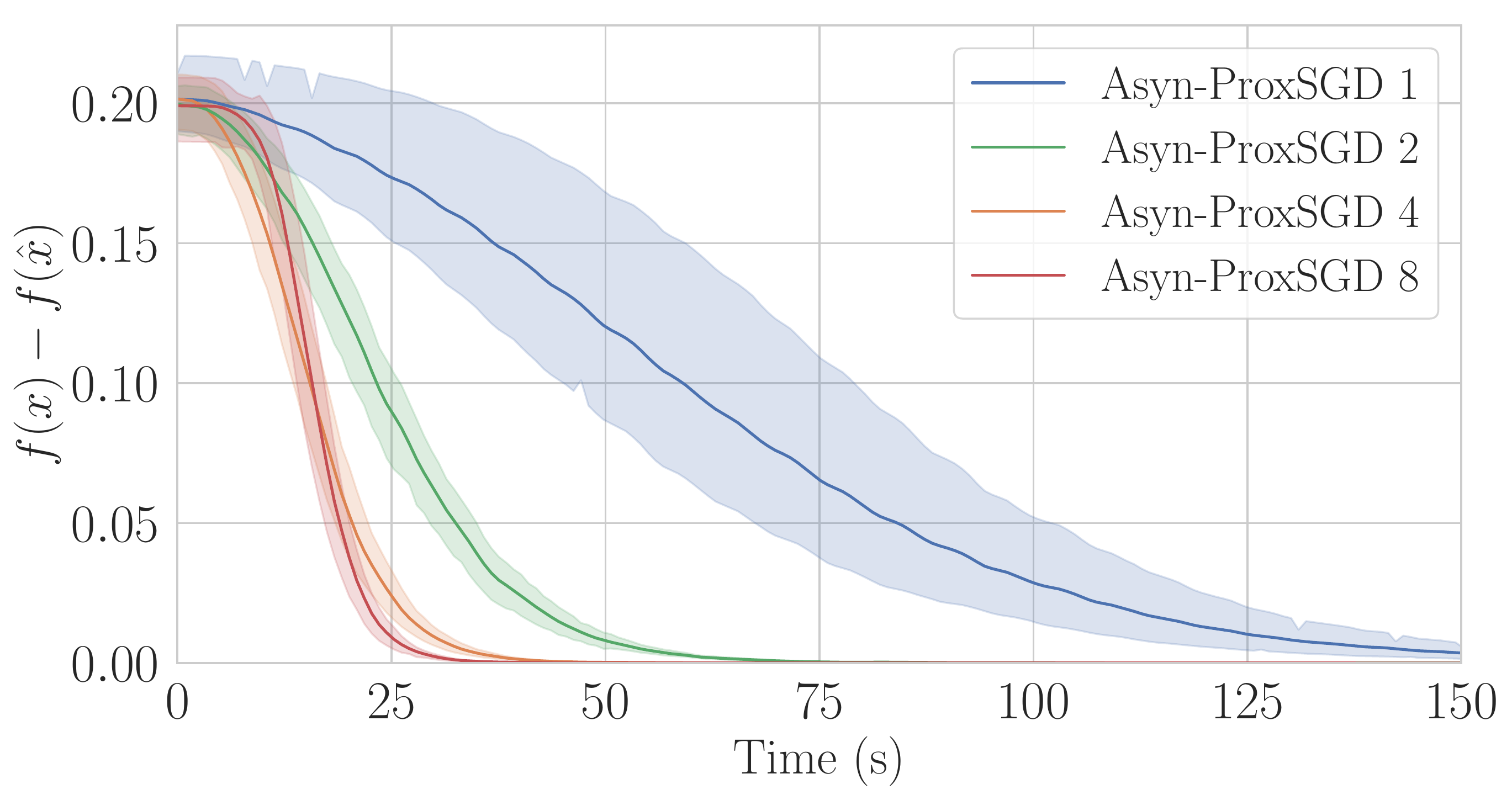}
    \label{fig:obj3-a9a}
  }
  \subfigure[\texttt{mnist}]{
    \includegraphics[height=1.6in]{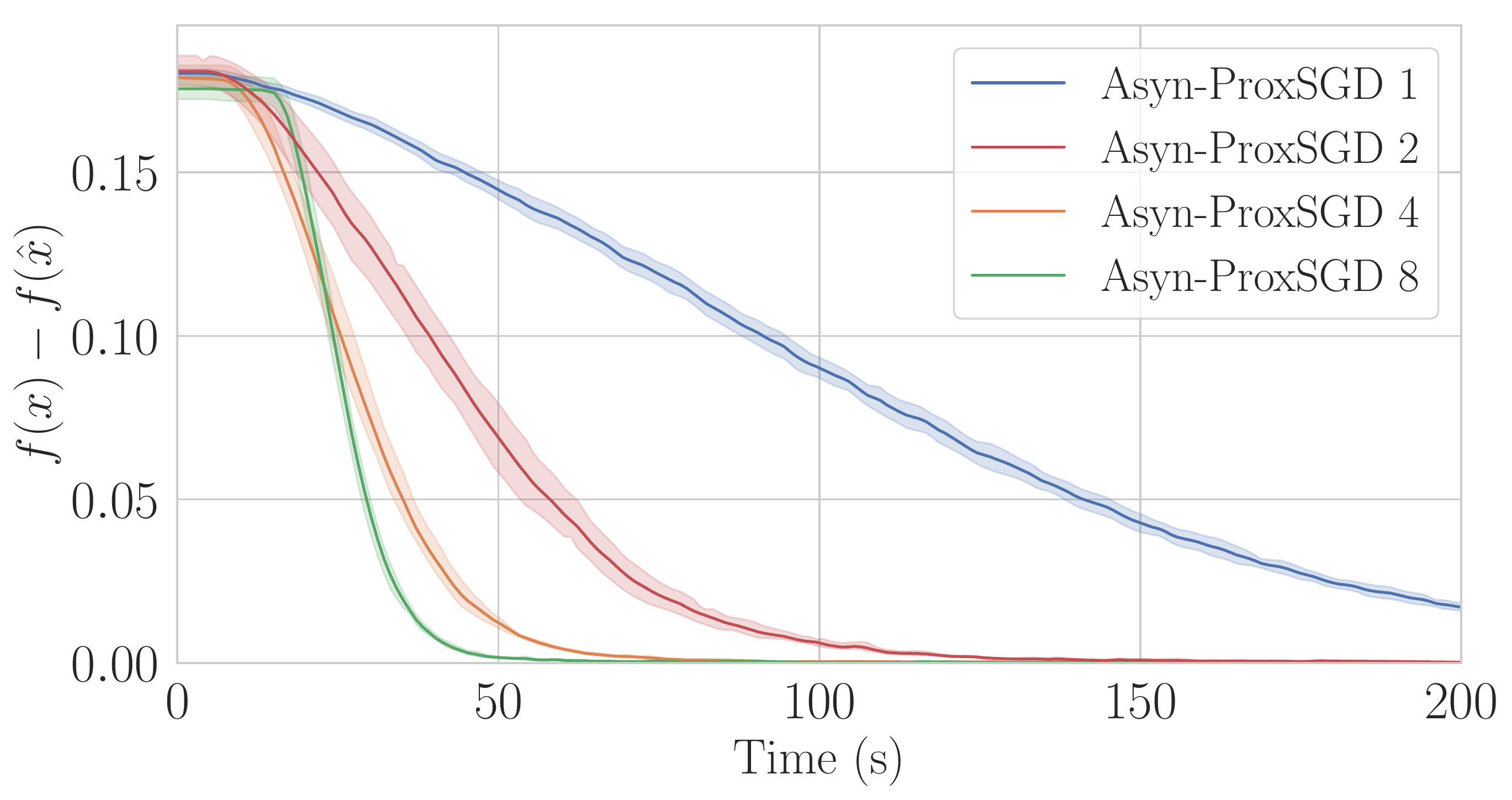}
    \label{fig:obj3-mnist}
  }
  \vspace{-1mm}
  \subfigure[\texttt{a9a}]{
    \includegraphics[height=1.6in]{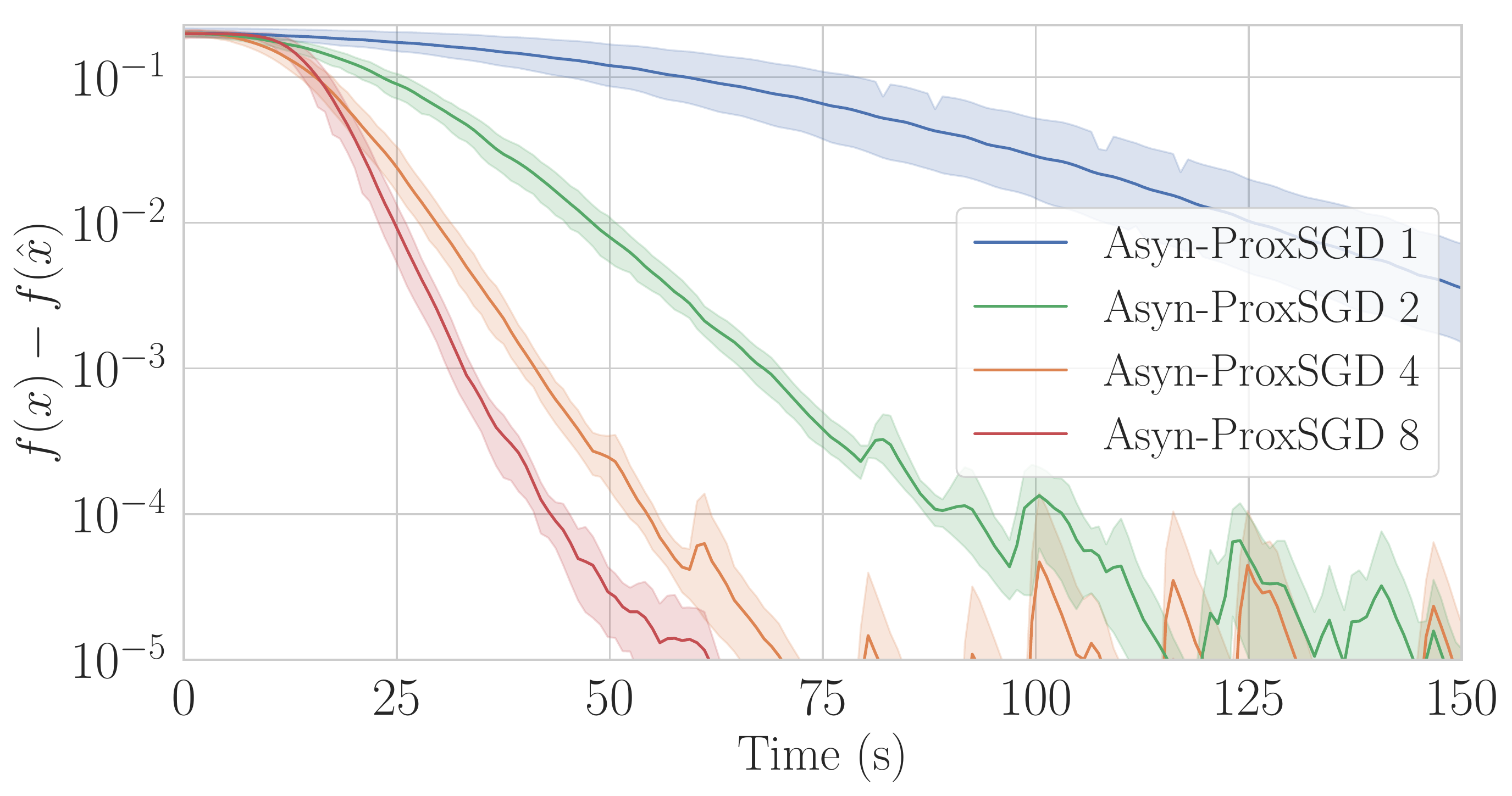}
    \label{fig:obj4-a9a}
  }
  \subfigure[\texttt{mnist}]{
    \includegraphics[height=1.6in]{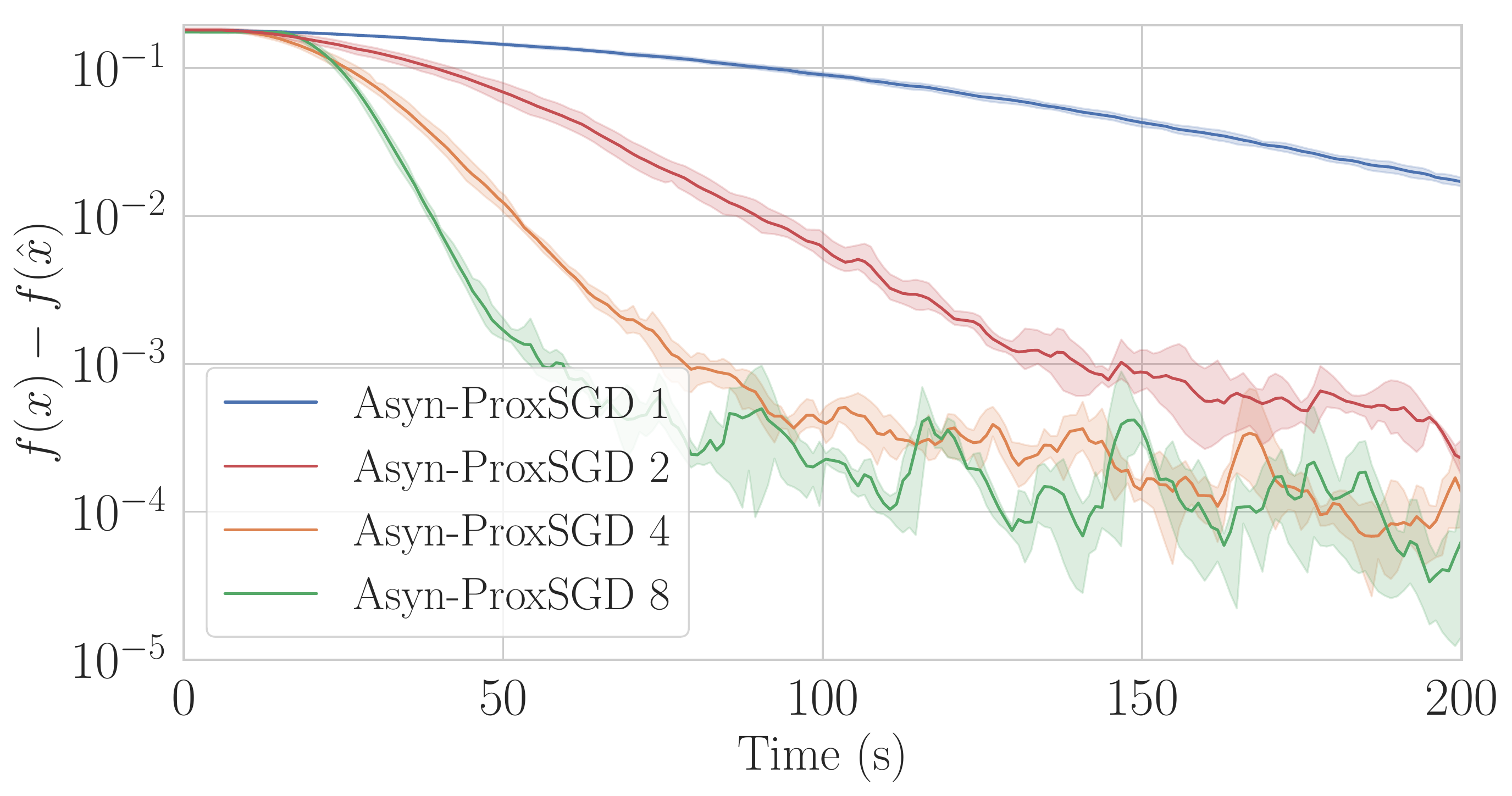}
    \label{fig:obj4-mnist}
  }
  \caption{Performance of ProxGD and Async-ProxSGD on \texttt{a9a} (left) and \texttt{mnist} (right) datasets. Here the x-axis represents the actual running time, and the y-axis is the function suboptimality. Note all values on the y-axis are normalized by $n$.}
  \label{fig:simu_time_loss}
\end{figure*}

\textbf{Results:}
Empirically, Assumption~\ref{asmp:bound1} (bounded delays) is observed to hold for this cluster. For our proposed Asyn-ProxSGD algorithm, we are particularly interested in the speedup in terms of iterations and running time.
In particular, if we need $T_1$ iterations (with $T_1$ sample gradients processed by the server) to achieve a certain suboptimality level using one worker, and $T_p$ iterations (with $T_p$ sample gradients processed by the server) to achieve the same suboptimality with $p$ workers, the iteration speedup is defined as $p \times T_1 / T_p$ \citep{lian2015asynchronous}. Note that all iterations are counted on the server side, i.e., how many sample gradients are processed by the server. On the other hand, the running time speedup is defined as the ratio between the running time of using one worker and that of using $p$ workers to achieve the same suboptimality. 

The iteration and running time speedups on both datasets are shown in Fig.~\ref{fig:simu_grad_loss} and
Fig.~\ref{fig:simu_time_loss}, respectively. Such speedups achieved at the suboptimality level of $10^{-3}$ are presented in Table~\ref{table:iter_speedup} and~\ref{table:time_speedup}. We observe that nearly linear speedup can be achieved, although there is a loss of efficiency due to communication as the number workers increases. 

\begin{table}[]
\centering
\caption{Iteration speedup and time speedup of Asyn-ProxSGD at the suboptimality level $10^{-3}$. (\texttt{a9a})}
\label{table:iter_speedup}
\begin{tabular}{c|cccc}
\specialrule{.1em}{.05em}{.05em}
Workers           &  1       &  2       &  4       &  8       \\ \hline
Iteration Speedup & 1.000    & 1.982    & 3.584    & 5.973    \\
Time Speedup      & 1.000    & 2.219    & 3.857    & 5.876    \\
\specialrule{.1em}{.05em}{.05em}
\end{tabular}
\end{table}

\begin{table}[]
\centering
\caption{Iteration speedup and time speedup of Asyn-ProxSGD at the suboptimality level $10^{-3}$. (\texttt{mnist})}
\label{table:time_speedup}
\begin{tabular}{c|cccc}
\specialrule{.1em}{.05em}{.05em}
Workers           &  1       &  2       &  4       &  8       \\ \hline
Iteration Speedup & 1.000    & 2.031    & 3.783    & 7.352    \\
Time Speedup      & 1.000    & 2.285    & 4.103    & 5.714    \\
\specialrule{.1em}{.05em}{.05em}
\end{tabular}
\end{table}

\section{Related Work}
\label{sec:related}

Stochastic optimization problems have been studied since the seminal work in 1951 \citep{robbins1951stochastic}, in which a classical stochastic approximation algorithm is proposed for solving a class of strongly convex problems. Since then, a series of studies on stochastic programming have focused on convex problems using SGD \citep{bottou1991stochastic,nemirovskii1983problem,moulines2011non}. The convergence rates of SGD for convex and strongly convex problems are known to be $O(1/\sqrt{K})$ and $O(1/K)$, respectively. For nonconvex optimization problems using SGD, Ghadimi and Lan \citep{ghadimi2013stochastic} proved an ergodic convergence rate of $O(1/\sqrt{K})$, which is consistent with the convergence rate of SGD for convex problems. 


When $h(\cdot)$ in \eqref{eq:original} is not necessarily smooth, there are other methods to handle the nonsmoothness. One approach is closely related to mirror descent stochastic approximation, e.g., \citep{nemirovski2009robust,lan2012optimal}. Another approach is based on proximal operators \citep{parikh2014proximal}, and is often referred to as the \emph{proximal stochastic gradient descent} (ProxSGD) method. Duchi et al. \citep{duchi2009efficient} prove that under a diminishing learning rate $\eta_k=1/(\mu k)$ for $\mu$-strongly convex objective functions, ProxSGD can achieve a convergence rate of $O(1/\mu K)$. For a nonconvex problem like \eqref{eq:original}, rather limited studies on ProxSGD exist so far. 
The closest approach to the one we consider here is \citep{ghadimi2016mini}, in which the convergence analysis is based on the assumption of an increasing minibatch size. Furthermore, Reddi et al. \citep{reddi2016proximal} prove convergence for nonconvex problems under a constant minibatch size, yet relying on additional mechanisms for variance reduction. We fill the gap in the literature by providing convergence rates for ProxSGD under constant batch sizes without variance reduction.


To deal with big data, asynchronous parallel optimization algorithms have been heavily studied. Recent work on asynchronous parallelism is mainly limited to the following categories: stochastic gradient descent for smooth optimization, e.g., \citep{recht2011hogwild,agarwal2011distributed,lian2015asynchronous,pan2016cyclades,mania2017} and deterministic ADMM, e.g. \citep{zhang2014asynchronous,hong2017distributed}. A non-stochastic asynchronous ProxSGD algorithm is presented by~\citep{li2014communication}, which however did not provide convergence rates for nonconvex problems.



\section{Concluding Remarks}
\label{sec:conclude}

In this paper, we study asynchronous parallel implementations of stochastic proximal gradient methods for solving nonconvex optimization problems, with convex yet possibly nonsmooth regularization.
However, compared to asynchronous parallel stochastic gradient descent (Asyn-SGD), which is targeting smooth optimization, the understanding of the convergence and speedup behavior of stochastic algorithms for the nonsmooth regularized optimization problems is quite limited, especially when the objective function is nonconvex. To fill this gap, we propose an asynchronous proximal stochastic gradient descent (Asyn-ProxSGD) algorithm with convergence rates provided for nonconvex problems. Our theoretical analysis suggests that the same order of convergence rate can be achieved for asynchronous ProxSGD for nonsmooth problems as for the asynchronous SGD, under constant minibatch sizes, without making additional assumptions on variance reduction. And a linear speedup is proven to be achievable for both asynchronous ProxSGD when the number of workers is bounded by $O({K}^{1/4})$. 

\bibliographystyle{abbrvnat}
\bibliography{src/main}

\newpage
\appendix



\section{Auxiliary Lemmas}

\begin{lemma}[\citep{ghadimi2016mini}]
    For all $y \gets \prox{\eta h}(x - \eta g)$, we have:
    \begin{equation}
        \innerprod{g, y-x} + (h(y) - h(x)) \leq -\frac{\norm{y-x}_2^2}{\eta}.
    \end{equation}
    \label{lem:lem-1}
\end{lemma}
Due to slightly different notations and definitions in \citep{ghadimi2016mini}, we provide a proof here for completeness. We refer readers to \citep{ghadimi2016mini} for more details.
\begin{proof}
    By the definition of proximal function, there exists a $p \in \partial h(y)$ such that:
    \begin{equation*}
        \begin{split}
            \innerprod{g + \frac{y-x}{\eta} + p, x - y} &\geq 0, \\
            \innerprod{g, x-y} &\geq \frac{1}{\eta} \innerprod{y-x, y-x} + \innerprod{p, y-x} \\
            \innerprod{g, x-y} + (h(x) - h(y)) &\geq \frac{1}{\eta} \norm{y-x}_2^2,
        \end{split}
    \end{equation*}
    which proves the lemma.
\end{proof}

\begin{lemma}[\citep{ghadimi2016mini}]
    For all $x,g,G \in \real{d}$, if $h: \real{d} \to \real{}$ is a convex function, we have
    \begin{equation}
        \norm{\prox{\eta h}(x - \eta G) - \prox{\eta h}(x - \eta g)} \leq \eta \norm{G-g}.
    \end{equation}
    \label{lem:lem-2}
\end{lemma}
\begin{proof}
    Let $y$ denote $\prox{\eta h}(x - \eta G)$ and $z$ denote $\prox{\eta h}(x - \eta g)$. By definition of the proximal operator, for all $u \in \real{d}$, we have
    \begin{align*}
        \innerprod{G + \frac{y-x}{\eta} + p, u-y} &\geq 0, \\
        \innerprod{g + \frac{z-x}{\eta} + q, u-z} &\geq 0,
    \end{align*}
    where $p \in \partial h(y)$ and $q \in \partial h(z)$. Let $z$ substitute $u$ in the first inequality and $y$ in the second one, we have
    \begin{align*}
        \innerprod{G + \frac{y-x}{\eta} + p, z-y} &\geq 0, \\
        \innerprod{g + \frac{z-x}{\eta} + q, y-z} &\geq 0.
    \end{align*}
    Then, we have
    \begin{align}
        \innerprod{G, z-y} &\geq \innerprod{\frac{y-x}{\eta}, y-z} + \innerprod{p, y-z}, \\
        &= \frac{1}{\eta}\innerprod{y-z, y-z} + \frac{1}{\eta}\innerprod{z-x, y-z} + \innerprod{p, y-z}, \\
        &\geq \frac{\norm{y-z}^2}{\eta} + \frac{1}{\eta}\innerprod{z-x, y-z} +h(y) - h(z),
         \label{eq:lem-2-1}
    \end{align}
    and
    \begin{align}
        \innerprod{g, y-z} &\geq \innerprod{\frac{z-x}{\eta} + q, z-y}, \\
        &= \frac{1}{\eta} \innerprod{z-x, z-y} + \innerprod{q, z-y} \\
        &\geq \frac{1}{\eta} \innerprod{z-x, z-y} + h(z) - h(y). \label{eq:lem-2-2}
    \end{align}
    By adding \eqref{eq:lem-2-1} and \eqref{eq:lem-2-2}, we obtain
    \begin{align*}
        \norm{G-g} \norm{z-y} \geq \innerprod{G-g, z-y} \geq \frac{1}{\eta} \norm{y-z}^2,
    \end{align*}
    which proves the lemma.
\end{proof}

\begin{lemma}[\citep{ghadimi2016mini}]
    For any $g_1$ and $g_2$, we have
    \begin{equation}
        \norm{P(x, g_1, \eta) - P(x, g_2, \eta)} \leq \norm{g_1 - g_2}.
    \end{equation}
    \label{lem:lem-3}
\end{lemma}
\begin{proof}
    It can be obtained by directly applying Lemma~\ref{lem:lem-2} and the definition of gradient mapping.
\end{proof}

\begin{lemma}[\citep{reddi2016proximal}]
    Suppose we define $y = \prox{\eta h} (x - \eta g)$ for some $g$. Then for $y$, the following inequality holds:
    \begin{equation}
    \begin{split}
        \Psi(y) \leq \Psi(z) + &\innerprod{y-z, \nabla f(x) - g} \\
        &+ \left( \frac{L}{2} - \frac{1}{2\eta} \right) \norm{y-x}^2
         + \left( \frac{L}{2} + \frac{1}{2\eta} \right) \norm{z-x}^2
         - \frac{1}{2\eta} \norm{y-z}^2,
    \end{split}
    \end{equation}
    for all $z$.
    \label{lem:grad_diff}
\end{lemma}

We recall and define some notations for convergence analysis in the subsequent. We denote $\tilde{G}_k$ as the average of \emph{delayed} stochastic gradients and $\tilde{g}_k$ as the average of \emph{delayed} true gradients, respectively:
\begin{align*}
    \tilde{G}_k &:= \frac{1}{N}\sum_{i=1}^N \nabla F(x_{t(k,i)}; \xi_{t(k,i), i}) \\
    \tilde{g}_k &:= \frac{1}{N}\sum_{i=1}^N \nabla f(x_{t(k,i)}).
\end{align*}
Moreover, we denote $\delta_k := \tilde{g}_k - \tilde{G}_k$ as the difference between these two differences.

\section{Convergence analysis for Asyn-ProxSGD}

\subsection{Milestone lemmas}
We put some key results of convergence analysis as milestone lemmas listed below, and the detailed proof is listed in \ref{subsec:milestone1}.
\begin{lemma}[Decent Lemma]
    \begin{equation}
        \mathbb{E}[\Psi(x_{k+1}) \leq \mathbb{E}[\Psi(x_k)|\mathcal{F}_k] - \frac{\eta_k - 4L\eta_k^2}{2} \norm{P(x_k, \tilde{g}_k, \eta_k)}^2 
      + \frac{\eta_k}{2}\norm{g_k-\tilde{g}_k}^2  + \frac{L\eta_k^2}{N}\sigma^2.
      \label{eq:desc_1}
    \end{equation}
    \label{lem:desc_1}
\end{lemma}

\begin{lemma}
    Suppose we have a sequence $\{x_k\}$ by Algorithm~\ref{alg:apsgd-global}, then we have:
    \begin{align}
        \mathbb{E}[\norm{x_k - x_{k-\tau}}^2 ] 
        \leq \left(\frac{2 \tau}{N} \sum_{l=1}^{\tau}\eta_{k-l}^2\right)\sigma^2 + 2\Norm{\sum_{l=1}^{\tau} \eta_{k-l} P(x_{k-l}, \tilde{g}_{k-l}, \eta_{k-l}) }^2.
    \label{eq:xk_diff_bound}
    \end{align}
    for all $\tau > 0$.
    \label{lem:xk_diff_1}
\end{lemma}

\begin{lemma}
    Suppose we have a sequence $\{x_k\}$ by Algorithm~\ref{alg:apsgd-global}, , then we have:
    \begin{equation}
        \mathbb{E}[\norm{g_k-\tilde{g}_k}^2]
        \leq \left( \frac{2L^2T}{N} \sum_{l=1}^{T}\eta_{k-l}^2 \right) \sigma^2 + 2L^2T\sum_{l=1}^{T}\eta_{k-l}^2 \norm{ P(x_{k-l}, \tilde{g}_{k-l}, \eta_{k-l}) }^2.
    \end{equation}
    \label{lem:gk_diff_1}
\end{lemma}

\subsection{Proof of Theorem~\ref{thm:aspg_convergence}}
\begin{proof}
From the fact $2\norm{a}^2 + 2\norm{b}^2 \geq \norm{a+b}^2$, we have
\begin{align*}
    \norm{P(x_k, \tilde{g}_k, \eta_k)}^2 + \norm{g_k - \tilde{g}_k}^2
    &\geq \norm{P(x_k, \tilde{g}_k, \eta_k)}^2 + \norm{P(x_k, g_k, \eta_k) - P(x_k, \tilde{g}_k, \eta_k)}^2 \\
    &\geq \frac{1}{2} \norm{P(x_k, g_k, \eta_k)}^2,
\end{align*}
which implies that 
\begin{align*}
    \norm{P(x_k, \tilde{g}_k, \eta_k)}^2 &\geq \frac{1}{2} \norm{P(x_k, g_k, \eta_k)}^2 - \norm{g_k - \tilde{g}_k}^2.
\end{align*}
We start the proof from Lemma~\ref{lem:desc_1}. According to our condition of $\eta \leq \frac{1}{16L}$, we have $8L\eta_k^2 - \eta < 0$ and therefore
{
\begin{equation*}
    \begin{split}
    &\quad\ \mathbb{E}[\Psi(x_{k+1})|\mathcal{F}_k] \\
    &\leq \mathbb{E}[\Psi(x_k)|\mathcal{F}_k] + \frac{\eta_k}{2}\norm{g_k-\tilde{g}_k}^2 
      + \frac{4L\eta_k^2-\eta_k}{2} \norm{P(x_k, \tilde{g}_k, \eta_k)}^2 + \frac{L\eta_k^2}{N}\sigma^2 \\
    &= \mathbb{E}[\Psi(x_k)|\mathcal{F}_k] + \frac{\eta_k}{2}\norm{g_k-\tilde{g}_k}^2 
      + \frac{8L\eta_k^2-\eta_k}{4} \norm{P(x_k, \tilde{g}_k, \eta_k)}^2 - \frac{\eta_k}{4}\norm{P(x_k, \tilde{g}_k, \eta_k)}^2 + \frac{L\eta_k^2}{N}\sigma^2 \\
    &\leq \mathbb{E}[\Psi(x_k)|\mathcal{F}_k] + \frac{\eta_k}{2}\norm{g_k-\tilde{g}_k}^2 
      + \frac{L\eta_k^2}{N}\sigma^2 + \frac{8L\eta_k^2-\eta_k}{4} (\frac{1}{2} \norm{P(x_k, g_k, \eta_k)}^2 - \norm{g_k - \tilde{g}_k}^2) \\
    &\quad\ - \frac{\eta_k}{4}\norm{P(x_k, \tilde{g}_k, \eta_k)}^2 \nonumber \\
    &\leq \mathbb{E}[\Psi(x_k)|\mathcal{F}_k] - \frac{\eta_k-8L\eta_k^2}{8} \norm{P(x_k, g_k, \eta_k)}^2 
      + \frac{3\eta_k}{4}\norm{g_k-\tilde{g}_k}^2  - \frac{\eta_k}{4}\norm{P(x_k, \tilde{g}_k, \eta_k)}^2 + \frac{L\eta_k^2}{N}\sigma^2.
\end{split}
\end{equation*}
}
Apply Lemma~\ref{lem:gk_diff_1} we have
\begin{equation*}
\begin{split}
    &\quad\ \mathbb{E}[\Psi(x_{k+1})|\mathcal{F}_k] \\
    &\leq \mathbb{E}[\Psi(x_k)|\mathcal{F}_k] - \frac{\eta_k-8L\eta_k^2}{8} \norm{P(x_k, g_k, \eta_k)}^2 
      + \frac{L\eta_k^2}{N}\sigma^2 - \frac{\eta_k}{4}\norm{P(x_k, \tilde{g}_k, \eta_k)}^2 \\
        &\quad + \frac{3\eta_k}{4}\left( \frac{2L^2T}{N} \sum_{l=1}^{T}\eta_{k-l}^2 \sigma^2 + 2L^2T\sum_{l=1}^{T}\eta_{k-l}^2 \norm{ P(x_{k-l}, \tilde{g}_{k-l}, \eta_{k-l}) }^2 \right) \\
    &= \mathbb{E}[\Psi(x_k)|\mathcal{F}_k] - \frac{\eta_k-8L\eta_k^2}{8} \norm{P(x_k, g_k, \eta_k)}^2 
      + \left(\frac{L\eta_k^2}{N} + \frac{3\eta_kL^2T}{2N} \sum_{l=1}^{T}\eta_{k-l}^2 \right) \sigma^2 \\
        &\quad - \frac{\eta_k}{4}\norm{P(x_k, \tilde{g}_k, \eta_k)}^2 
        + \frac{3\eta_kL^2T}{2} \sum_{l=1}^{T}\eta_{k-l}^2 \norm{ P(x_{k-l}, \tilde{g}_{k-l}, \eta_{k-l}) }^2.
\end{split}
\end{equation*}
By taking telescope sum, we have
{
\begin{align*}
    &\quad\ \mathbb{E}[\Psi(x_{K+1})|\mathcal{F}_K] \\
    &\leq \Psi(x_1)
        - \sum_{k=1}^K \frac{\eta_k-8L\eta_k^2}{8} \norm{P(x_k, g_k, \eta_k)}^2 
        - \sum_{k=1}^K \left( \frac{\eta_k}{4}- \frac{3\eta_k^2L^2T}{2}\sum_{l=1}^{l_k}\eta_{k+l} \right) \norm{ P(x_{k}, \tilde{g}_{k}, \eta_{k}) }^2 \\
    &\quad + \sum_{k=1}^K \left(\frac{L\eta_k^2}{N} + \frac{3\eta_kL^2T}{2N} \sum_{l=1}^{T}\eta_{k-l}^2 \right) \sigma^2 
\end{align*}
}
where $l_k := \min(k+T-1, K)$, and we have
{
\begin{align*}
    &\quad\ \sum_{k=1}^K \frac{\eta_k-8L\eta_k^2}{8} \norm{P(x_k, g_k, \eta_k)}^2 \\
    &\leq \Psi(x_1) - \mathbb{E}[\Psi(x_{K+1})|\mathcal{F}_K] - \sum_{k=1}^K \left( \frac{\eta_k}{4}- \frac{3\eta_k^2L^2T}{2}\sum_{l=1}^{l_k}\eta_{k+l} \right) \norm{ P(x_{k}, \tilde{g}_{k}, \eta_{k}) }^2 \\
    &\quad + \sum_{k=1}^K \left(\frac{L\eta_k^2}{N} + \frac{3\eta_kL^2T}{2N} \sum_{l=1}^{T}\eta_{k-l}^2 \right) \sigma^2.
\end{align*}
}
When $6\eta_kL^2T \sum_{l=1}^{T}\eta_{k+l} \leq 1$ for all $k$ as the condition of Theorem~\ref{thm:aspg_convergence}, we have
\begin{align*}
    &\quad\ \sum_{k=1}^K \frac{\eta_k-8L\eta_k^2}{8} \norm{P(x_k, g_k, \eta_k)}^2 \\
    &\leq \Psi(x_1) - \mathbb{E}[\Psi(x_{K+1})|\mathcal{F}_K] 
        + \sum_{k=1}^K \left(\frac{L\eta_k^2}{N} + \frac{3\eta_kL^2T}{2N} \sum_{l=1}^{T}\eta_{k-l}^2 \right) \sigma^2 \\
        &\leq \Psi(x_1) - F^* + \sum_{k=1}^K \left(\frac{L\eta_k^2}{N} + \frac{3\eta_kL^2T}{2N} \sum_{l=1}^{T}\eta_{k-l}^2 \right) \sigma^2,
\end{align*}
which proves the theorem.
\end{proof}

\subsection{Proof of Corollary~\ref{corr:aspg_convergence}}
\begin{proof}
From the condition of Corollary, we have
\begin{align*}
    \eta \leq \frac{1}{16L(T+1)^2}.
\end{align*}
It is clear that the above inequality also satisfies the condition in Theorem~\ref{thm:aspg_convergence}. By doing so, we can have
Furthermore, we have
\begin{align*}
    \frac{3LT^2\eta}{2} &\leq \frac{3LT^2}{2}\cdot \frac{1}{16L(T+1)^2} \leq 1, \\
    \frac{3L^2T^2\eta^3}{2} &\leq L\eta^2.
\end{align*}
Since $\eta \leq \frac{1}{16L}$, we have $2-16L\eta^2 \geq 1$ and thus
\begin{align*}
    \frac{8}{\eta - 8L\eta^2} = \frac{16}{\eta(2-16L\eta^2)} \leq \frac{16}{\eta}.
\end{align*}
Following Theorem~\ref{thm:aspg_convergence} and the above inequality, we have
\begin{align*}
    &\quad\ \frac{1}{K}\sum_{k=1}^K \mathbb{E}[\norm{P(x_k, g_k, \eta_k)}^2] \\
    &\leq \frac{16(\Psi(x_1) - \Psi(x_*))}{K\eta} + 16\left(\frac{L\eta^2}{N} + \frac{3\eta L^2T}{2N} \sum_{l=1}^{T}\eta^2 \right) \frac{K\sigma^2}{K\eta} \\
    &= \frac{16(\Psi(x_1) - \Psi(x_*)}{K\eta} + 16\left(\frac{L\eta^2}{N} + \frac{3L^2T^2 \eta^3}{2N} \right) \frac{\sigma^2}{\eta} \\
    &\leq \frac{16(\Psi(x_1) - \Psi(x_*))}{K\eta} + \frac{32L\eta^2}{N}\cdot \frac{\sigma^2}{\eta} \\
    &= \frac{16(\Psi(x_1) - \Psi(x_*))}{K\eta} + \frac{32L\eta \sigma^2}{N} \\
    &= 32\sqrt{\frac{2(\Psi(x_1) - \Psi(x_*))L\sigma^2}{KN}},
\end{align*}
which proves the corollary.
\end{proof}

\subsection{Proof of milestone lemmas}
\label{subsec:milestone1}
\begin{proof}[Proof of Lemma~\ref{lem:desc_1}]
Let $\bar{x}_{k+1} = \prox{\eta_k h}(x_k - \eta_k \tilde{g}_k)$ and apply Lemma~\ref{lem:grad_diff}, we have
\begin{equation}
    \begin{split}
        \Psi(x_{k+1}) &\leq \Psi(\bar{x}_{k+1}) + \innerprod{x_{k+1}-\bar{x}_{k+1}, \nabla f(x_k) - \tilde{G}_k} + \left( \frac{L}{2} - \frac{1}{2\eta_k} \right) \norm{x_{k+1}-x_k}^2\\
        &\quad\ + \left( \frac{L}{2} + \frac{1}{2\eta_k} \right) \norm{\bar{x}_{k+1}-x_k}^2 - \frac{1}{2\eta_k} \norm{x_{k+1}-\bar{x}_{k+1}}^2.
    \end{split}
    \label{eq:yk_baryk_async}
\end{equation}
Now we turn to bound $\Psi(\bar{x}_{k+1})$ as follows:
\begin{displaymath}
    \begin{split}
        f(\bar{x}_{k+1}) &\leq f(x_k) + \innerprod{\nabla f(x_k), \bar{x}_{k+1} - x_k} + \frac{L}{2}\norm{\bar{x}_{k+1} - x_k}^2 \\
        &= f(x_k) + \innerprod{g_k, \bar{x}_{k+1} - x_k} + \frac{\eta_k^2 L}{2}\norm{P(x_k, \tilde{g}_k, \eta_k)}^2 \\
        &= f(x_k) + \innerprod{\tilde{g}_k, \bar{x}_{k+1} - x_k}+ \innerprod{g_k - \tilde{g}_k, \bar{x}_{k+1} - x_k} + \frac{\eta_k^2 L}{2}\norm{P(x_k, \tilde{g}_k, \eta_k)}^2 \\
        &= f(x_k) - \eta_k \innerprod{\tilde{g}_k, P(x_k, \tilde{g}_k, \eta_k)} + \innerprod{g_k - \tilde{g}_k, \bar{x}_{k+1} - x_k} + \frac{\eta_k^2 L}{2}\norm{P(x_k, \tilde{g}_k, \eta_k)}^2 \\
        &\leq  f(x_k) - [\eta_k \norm{P(x_k, \tilde{g}_k, \eta_k)}^2 + h(\bar{x}_{k+1}) - h(x_k)] + \innerprod{g_k - \tilde{g}_k, \bar{x}_{k+1} - x_k} \\
        &\quad\ + \frac{\eta_k^2 L}{2}\norm{P(x_k, \tilde{g}_k, \eta_k)}^2,
    \end{split}
\end{displaymath}
where the last inequality follows from Lemma~\ref{lem:lem-1}. By rearranging terms on both sides, we have
\begin{equation}
    \Psi(\bar{x}_{k+1}) \leq \Psi(x_k) - (\eta_k - \frac{\eta_k^2 L}{2}) \norm{P(x_k, \tilde{g}_k, \eta_k)}^2 + \innerprod{g_k - \tilde{g}_k, \bar{x}_{k+1} - x_k}
    \label{eq:baryk_xk_async}
\end{equation}
Taking the summation of \eqref{eq:yk_baryk_async} and \eqref{eq:baryk_xk_async}, we have
\begin{displaymath}
\begin{split}
    &\quad\ \Psi(x_{k+1}) \\
    &\leq \Psi(x_k) + \innerprod{x_{k+1}-\bar{x}_{k+1}, \nabla f(x_k) - \tilde{G}_k} \\
    &\quad + \left( \frac{L}{2} - \frac{1}{2\eta_k} \right) \norm{x_{k+1}-x_k}^2
     + \left( \frac{L}{2} + \frac{1}{2\eta_k} \right) \norm{\bar{x}_{k+1}-x_k}^2- \frac{1}{2\eta_k} \norm{x_{k+1}-\bar{x}_{k+1}}^2 \\
    &\quad - (\eta_k - \frac{\eta_k^2 L}{2}) \norm{P(x_k, \tilde{g}_k, \eta_k)}^2 + \innerprod{g_k - \tilde{g}_k, \bar{x}_{k+1} - x_k} \\
    &= \Psi(x_k) + \innerprod{x_{k+1}-x_{k}, g_k - \tilde{g}_k} + \innerprod{x_{k+1}-\bar{x}_{k+1}, \delta_k} \\
    &\quad + \left( \frac{L\eta_k^2}{2} - \frac{\eta_k}{2} \right) \norm{P(x_k, \tilde{G}_k, \eta_k)}^2
     + \left( \frac{L\eta_k^2}{2} + \frac{\eta_k}{2} \right) \norm{P(x_k, \tilde{g}_k, \eta_k)}^2 \\
    &\quad - \frac{1}{2\eta_k} \norm{x_{k+1}-\bar{x}_{k+1}}^2  - (\eta_k - \frac{\eta_k^2 L}{2}) \norm{P(x_k, \tilde{g}_k, \eta_k)}^2 \\
    &= \Psi(x_k) + \innerprod{x_{k+1}-x_k, g_k - \tilde{g}_k} + \innerprod{x_{k+1}-\bar{x}_{k+1}, \delta_k}
     + \frac{L\eta_k^2 - \eta_k}{2} \norm{P(x_k, \tilde{G}_k, \eta_k)}^2 \\
    &\quad + \frac{2L\eta_k^2-\eta_k}{2} \norm{P(x_k, \tilde{g}_k, \eta_k)}^2 - \frac{1}{2\eta_k} \norm{x_{k+1}-\bar{x}_{k+1}}^2
\end{split}
\end{displaymath}
By taking the expectation on condition of filtration $\mathcal{F}_k$ and according to Assumption~\ref{asmp:unbias_grad}, we have
\begin{equation}
\begin{split}
&\quad\ \mathbb{E}[\Psi(x_{k+1})|\mathcal{F}_k] \\
&\leq \mathbb{E}[\Psi(x_k)|\mathcal{F}_k] + \mathbb{E}[\innerprod{x_{k+1}-x_k, g_k-\tilde{g}_k}|\mathcal{F}_k] + \frac{L\eta_k^2 - \eta_k}{2} \mathbb{E}[\norm{P(x_k, \tilde{G}_k, \eta_k)}^2|\mathcal{F}_k] \\
&\quad + \frac{2L\eta_k^2-\eta_k}{2} \norm{P(x_k, \tilde{g}_k, \eta_k)}^2 - \frac{1}{2\eta_k} \norm{x_{k+1}-\bar{x}_{k+1}}^2.
\end{split}
\end{equation}
Therefore, we have
\begin{equation*}
\begin{split}
    &\quad\ \mathbb{E}[\Psi(x_{k+1})|\mathcal{F}_k] \\
    &\leq \mathbb{E}[\Psi(x_k)|\mathcal{F}_k] + \mathbb{E}[\innerprod{x_{k+1}-x_k, g_k-\tilde{g}_k}|\mathcal{F}_k] + \frac{L\eta_k^2 - \eta_k}{2} \mathbb{E}[\norm{P(x_k, \tilde{G}_k, \eta_k)}^2|\mathcal{F}_k] \\
      &\quad + \frac{2L\eta_k^2-\eta_k}{2} \norm{P(x_k, \tilde{g}_k, \eta_k)}^2 - \frac{1}{2\eta_k} \norm{x_{k+1}-\bar{x}_{k+1}}^2 \\
    &\leq \mathbb{E}[\Psi(x_k)|\mathcal{F}_k] + \frac{\eta_k}{2}\norm{g_k-\tilde{g}_k}^2 + \frac{L\eta_k^2}{2} \mathbb{E}[\norm{P(x_k, \tilde{G}_k, \eta_k)}^2|\mathcal{F}_k]
      + \frac{2L\eta_k^2-\eta_k}{2} \norm{P(x_k, \tilde{g}_k, \eta_k)}^2  \\
    &\leq \mathbb{E}[\Psi(x_k)|\mathcal{F}_k] - \frac{\eta_k - 4L\eta_k^2}{2} \norm{P(x_k, \tilde{g}_k, \eta_k)}^2 
      + \frac{\eta_k}{2}\norm{g_k-\tilde{g}_k}^2  + \frac{L\eta_k^2}{N}\sigma^2
\end{split}
\end{equation*}
\end{proof}

\begin{proof}[Proof of Lemma~\ref{lem:xk_diff_1}]
Following the definition of $x_k$ from Algorithm~\ref{alg:apsgd-global}, we have
\begin{equation*}
\begin{split}
    &\quad\ \norm{x_k - x_{k-\tau}}^2 \\
    &= \Norm{\sum_{l=1}^{\tau}x_{k-l} - x_{k-l+1}}^2 \\
    &= \Norm{\sum_{l=1}^{\tau} \eta_{k-l}  P(x_{k-l}, \tilde{G}_{k-l}, \eta_{k-l})}^2 \\
    &= 2 \Norm{\sum_{l=1}^{\tau} \eta_{k-l} [P(x_{k-l}, \tilde{G}_{k-l}, \eta_{k-l})-P(x_{k-l}, \tilde{g}_{k-l}, \eta_{k-l})]}^2 + 2\Norm{\sum_{l=1}^{\tau} \eta_{k-l} P(x_{k-l}, \tilde{g}_{k-l}, \eta_{k-l}) }^2 \\
    &\leq 2\tau \sum_{l=1}^{\tau} \eta_{k-l}^2\norm{P(x_{k-l}, \tilde{G}_{k-l}, \eta_{k-l})-P(x_{k-l}, \tilde{g}_{k-l}, \eta_{k-l})}^2 + 2\Norm{\sum_{l=1}^{\tau} \eta_{k-l} P(x_{k-l}, \tilde{g}_{k-l}, \eta_{k-l}) }^2 \\
    &\leq 2\tau \sum_{l=1}^{\tau}\eta_{k-l}^2 \norm{\tilde{G}_{k-l}-\tilde{g}_{k-l}}^2
    + 2\Norm{\sum_{l=1}^{\tau} \eta_{k-l} P(x_{k-l}, \tilde{g}_{k-l}, \eta_{k-l}) }^2,
\end{split}
\end{equation*}
where the last inequality is from Lemma~\ref{lem:lem-3}. By taking the expectation on both sides, we have
\begin{equation*}
    \begin{split}
    \mathbb{E}[\norm{x_k - x_{k-\tau}}^2 ] 
    &\leq 2\tau \sum_{l=1}^{\tau} \eta_{k-l}^2\norm{\tilde{G}_{k-l}-\tilde{g}_{k-l}}^2 + 2\Norm{\sum_{l=1}^{\tau} \eta_{k-l} P(x_{k-l}, \tilde{g}_{k-l}, \eta_{k-l}) }^2 \\
    &\leq \frac{2 \tau}{N} \sigma^2 \sum_{l=1}^{\tau}\eta_{k-l}^2 + 2\Norm{\sum_{l=1}^{\tau} \eta_{k-l} P(x_{k-l}, \tilde{g}_{k-l}, \eta_{k-l}) }^2,
\end{split}
\end{equation*}
which proves the lemma.
\end{proof}

\begin{proof}[Proof of Lemma~\ref{lem:gk_diff_1}]
From Assumption~\ref{asmp:smooth} we have
\begin{align*}
    \norm{g_k - \tilde{g}_k}^2 = \Norm{\frac{1}{N}\sum_{i=1}^N g_k - \tilde{g}_{t(k,i)}}^2 
    \leq \frac{L^2}{N} \sum_{i=1}^N \norm{x_k - x_{k-\tau(k,i)}}^2.
\end{align*}
By applying Lemma~\ref{lem:xk_diff_1}, we have
\begin{align*}
    \mathbb{E}[\norm{x_k - x_{k-\tau(k,i)}}^2] &\leq \frac{2\tau(k,i)}{N} \sigma^2 \sum_{l=1}^{\tau(k,i)}\eta_{k-l}^2 + 2\Norm{\sum_{l=1}^{\tau(k,i)} \eta_{k-l} P(x_{k-l}, \tilde{g}_{k-l}, \eta_{k-l}) }^2.
\end{align*}
Therefore, we have
\begin{equation*}
\begin{split}
    \mathbb{E}[\norm{g_k-\tilde{g}_k}^2] &\leq \frac{L^2}{N} \sum_{i=1}^N \norm{x_k - x_{k-\tau(k,i)}}^2 \nonumber \\
    &\leq \frac{L^2}{N} \sum_{i=1}^N \left( \frac{2\tau(k,i)}{N} \sigma^2 \sum_{l=1}^{\tau(k,i)}\eta_{k-l}^2 + 2\tau(k,i)\sum_{l=1}^{\tau(k,i)}\eta_{k-l}^2 \norm{ P(x_{k-l}, \tilde{g}_{k-l}, \eta_{k-l}) }^2 \right) \nonumber \\
    &\leq \left( \frac{2L^2T}{N} \sum_{l=1}^{T}\eta_{k-l}^2 \right) \sigma^2 + 2L^2T\sum_{l=1}^{T}\eta_{k-l}^2 \norm{ P(x_{k-l}, \tilde{g}_{k-l}, \eta_{k-l}) }^2,
\end{split}
\end{equation*}
where the last inequality follows from and now we prove the lemma.
\end{proof}

\end{document}